%% file: main.tex
\documentclass{article}

\usepackage{subfigure}
\usepackage{multirow}
\usepackage[utf8]{inputenc} % allow utf-8 input
\usepackage[T1]{fontenc}    % use 8-bit T1 fonts
\usepackage{hyperref}       % hyperlinks
\usepackage{url}            % simple URL typesetting
\usepackage{booktabs}       % professional-quality tables
\usepackage{amsfonts}       % blackboard math symbols
\usepackage{nicefrac}       % compact symbols for 1/2, etc.
\usepackage{microtype}      % microtypography
\usepackage{graphicx}
\usepackage{braket}
\usepackage{amsmath}
\usepackage{color}
\usepackage{mathtools}
\usepackage[toc,page]{appendix}
\usepackage{amsthm}
\usepackage{float}
\usepackage{multicol}

\usepackage{enumitem}
\setlist[itemize]{leftmargin=5mm}

\usepackage{picture}
\usepackage{wrapfig}

\usepackage{cleveref}
\usepackage{bm}

\usepackage{hyperref}

\usepackage[accepted]{icml2019}

\icmltitlerunning{Breaking the Softmax Bottleneck via Learnable Monotonic Pointwise Non-linearities}

\newcommand{\h}{{\mathbf h}}
\newcommand{\x}{{\mathbf x}}
\newcommand{\w}{{\mathbf w}}
\newcommand{\g}{{\mathbf g}}
\newcommand{\e}{{\mathbf e}}
\newcommand{\vv}{{\mathbf v}}

\newcommand{\I}{{\mathbf I}}

\newcommand{\HH}{{\mathbf H}}
\newcommand{\A}{{\mathbf A}}
\newcommand{\W}{{\mathbf W}}
\newcommand{\G}{{\mathbf G}}

\newcommand{\E}{{\mathbb E}}
\newcommand{\M}{{\mathbf M}}
\newcommand{\R}{{\mathbb R}}
\newcommand{\C}{{\mathbf C}}
\newcommand{\B}{{\mathbf B}}

\newcommand{\U}{{\mathbf U}}

\newtheorem{theorem}{Theorem}
\newtheorem{corollary}{Corollary}[theorem]
\newtheorem{lemma}[theorem]{Lemma}

\theoremstyle{definition}

\DeclarePairedDelimiterX{\inp}[2]{\langle}{\rangle}{#1,#2}

\begin{document}

\twocolumn[
\icmltitle{Breaking the Softmax Bottleneck via Learnable Monotonic Pointwise Non-linearities}



\icmlsetsymbol{equal}{*}

\begin{icmlauthorlist}
\icmlauthor{Octavian-Eugen Ganea}{to}
\icmlauthor{Sylvain Gelly}{goo}
\icmlauthor{Gary B{\'e}cigneul}{to}
\icmlauthor{Aliaksei Severyn}{goorr}
\end{icmlauthorlist}

\icmlaffiliation{to}{Department of Computer Science, ETH Z\"urich, Switzerland}
\icmlaffiliation{goo}{Google Brain}
\icmlaffiliation{goorr}{Google Research}

\icmlcorrespondingauthor{Octavian-Eugen Ganea}{octavian.ganea@inf.ethz.ch}
\icmlcorrespondingauthor{Sylvain Gelly}{sylvaingelly@google.com}

\icmlkeywords{Softmax bottleneck, nlp, language models}

\vskip 0.3in
]



\printAffiliationsAndNotice{}  % leave blank if no need to mention equal contribution

\begin{abstract}
The Softmax function on top of a final linear layer is the de facto method to output probability distributions in neural networks. In many applications such as language models or text generation, this model has to produce distributions over large output vocabularies. Recently, this has been shown to have limited representational capacity due to its connection with the rank bottleneck in matrix factorization. However, little is known about the limitations of Linear-Softmax for quantities of practical interest such as cross entropy or mode estimation, a direction that we explore here. As an efficient and effective solution to alleviate this issue, we propose to learn parametric monotonic functions on top of the logits. We theoretically investigate the rank increasing capabilities of such monotonic functions. Empirically, our method improves in two different quality metrics over the traditional Linear-Softmax layer in synthetic and real language model experiments, adding little time or memory overhead, while being comparable to the more computationally expensive mixture of Softmaxes.
\end{abstract}

\input{introduction}

\input{problem_and_mos}

\input{sol_mon}

\input{experiments}

\input{conclusion}

\section*{Acknowledgements}
We thank Josip Djolonga for useful discussions and anonymous reviewers for suggestions. 

Octavian Ganea is funded by the Swiss National Science Foundation (SNSF) under grant agreement number 167176. Gary B\'ecigneul is funded by the Max Planck ETH Center for Learning Systems.

%%%%%%%%%%%%%%%%%%%%%%%%%%%%%%%%%%%%%%%%%%%%%%%%%%%%%%%%%%%%%%%%%%%%%%%%%%%%%%%
%\clearpage

\bibliography{main.bib}
\bibliographystyle{icml2019}

\newpage
\appendix

\input{appendix}

%%%%%%%%%%%%%%%%%%%%%%%%%%%%%%%%%%%%%%%%%%%%%%%%%%%%%%%%%%%%%%%%%%%%%%%%%%%%%%%
%%%%%%%%%%%%%%%%%%%%%%%%%%%%%%%%%%%%%%%%%%%%%%%%%%%%%%%%%%%%%%%%%%%%%%%%%%%%%%%

\end{document}

%% file: introduction.tex
\section{Introduction}
\label{intro}

Most of nowadays deep learning architectures produce a low dimensional data representation that is important both from a computational (parameter reduction) and generalization (less overfitting) perspective. The underlying assumption is that data lies on a small dimensional manifold. These  compressed representations are then used for classification or generation. In the discrete case, they are usually fed to a linear layer to produce the so-called "logits", followed by a Softmax function to output a probability distribution over the desired class labels. We refer to this setup as the \textbf{Linear-Softmax} layer. 

\par However, there are situations when the output vocabulary or class label set is (much) larger than the dimension of the data embedding. Typical examples are text models such as neural language models~\citep{zaremba2014recurrent} or sequence to sequence generative models~\citep{sutskever2014sequence,graves2013generating,pascanu2013difficulty} for problems such as machine translation~\citep{bahdanau2014neural,cho2014properties}, text summarization~\citep{chopra2016abstractive,rush2015neural} or conversational agents~\citep{vinyals2015neural}. These models need to approximate different distributions over the full large vocabulary of words generally of size $\Theta(10^5)$. Recent work~\citep{yang2017breaking,kanai2018sigsoftmax} has revealed that, in these cases, the Linear-Softmax layer has limited representational power. They show the connection between this problem and the classic low-rank matrix factorization framework, concluding that the rank deficiency prevents Linear-Softmax from exactly matching in representation almost all\footnote{Except a subset of measure 0.} probability distributions. 

\par To address the Softmax bottleneck issue,~\citep{yang2017breaking} propose to use a mixture of Softmax distributions (MoS) which achieves state-of-the-art language model perplexity on PennTreeBank (PTB) and WikiText2 (WT2) datasets. However, this method has no theoretical guarantees, being also several orders of magnitude more computationally expensive than Linear-Softmax as we show in \cref{experiments}. 

\par A different model was proposed by \citep{kanai2018sigsoftmax} that replace the exponential in Softmax by a product between exponential and sigmoid. This model called \textit{Sigsoftmax} can be reformulated as applying the pointwise nonlinearity $ss(x) := 2x - \log(1 + \exp(x))$ to the logits before they are fed to the Softmax function. Unfortunately, there is no theoretical guarantee that Sigsoftmax can convert low-rank to full-rank matrices. In addition, this model raises a few questions that we seek to explore here: i) what other non-linearities are suitable for breaking the Softmax bottleneck? ii) can we theoretically understand and guarantee which pointwise functions will break the Softmax bottleneck by increasing the matrix rank? iii) can we efficiently learn a good non-linearity for the task of interest jointly with the rest of the model? 

\par To address all aforementioned issues, we here propose to learn continuous increasing pointwise functions that would beneficially distort the logits before being fed to the Softmax layer. Our model is called \textbf{Linear-Monotonic-Softmax (LMS)}. We constrain our functions to be increasing since we want this transformation to be rank preserving, but we theoretically show that if there exists any other pointwise non-linearity to make our matrix full rank, then there is also an increasing continuous and differentiable function with the same property. %Moreover, this event happens almost always in practical settings. %, that is, we show almost surely the existence of such a function in the space of d-rank n-squared matrices.

% - The summarization of contributions are a bit hard to understand, especially the last line of the first point, the second to third line of the second point and the last few lines of the third point.

\par We now summarize our contributions:
\begin{itemize}
\item We propose the novel \textbf{Linear-Monotonic-Softmax (LMS)} model to break the Softmax bottleneck. It generalizes the approach in \citep{kanai2018sigsoftmax} by learning parametric pointwise increasing functions to optimally distort the logits before feeding them to the final Softmax. Theoretically, we investigate its power to alleviate the rank-deficiency causing the Softmax bottleneck. 
%\item On the theory side, we prove that almost always there exists such a monotonic function that would remove the rank-deficiency problem that generated the softmax bottleneck issue. To our knowledge, we are the first to give a probabilistic lower bound on the rank of a matrix changed by polynomial pointwise operators.\footnote{As opposed,~\citep{amini2012low} gives upper bounds for this class of models.}
\item We show insights into the Linear-Softmax bottleneck by analyzing some metrics of practical utility such as cross-entropy or mode matching. Theoretically, we connect the cross entropy minimization of this model and the principle of maximum entropy with linear constraints. 
\item Empirically, we show that, in a synthetic setting, Linear-Softmax and MoS~\citep{yang2017breaking} are (sometimes significantly) worse than LMS for cross-entropy minimization or mode matching. In the real task of language modeling, LMS applied to state-of-the-art models improves the test perplexity over vanilla Linear-Softmax~\citep{merity2017regularizing} and Sigsoftmax~\citep{kanai2018sigsoftmax} on standard benchmark datasets, with very little GPU memory or running time overhead,  being comparable to the significantly more expensive MoS model.
\end{itemize}

%% file: problem_and_mos.tex
\section{Language Modeling}
We first briefly explain a representative task for the Softmax bottleneck problem, namely language modeling (LM). However, this issue concerns any models that produce probability distributions over large output vocabularies. %Thus, we will keep the theoretical exposure rather generic.

Language models are the simplest fully unsupervised generative models for natural language text which are actively used to improve state-of-the-art results in various natural language processing tasks~\citep{peters2018deep,devlin2018bert}. Formally, assume we are given a vocabulary of words in a language $\mathcal{V} = \{x_1, \ldots, x_M\}$ and a text corpus represented as a sequence of words $\mathcal{X} = (x_{j_1}, \ldots, x_{j_N})$, where typically $N >> M$. We assume that this corpus is generated sequentially from a true conditional next-token distribution $P^*(X_i | X_{i-1}, \ldots, X_1) = P^*(X_i | C_i)$, where the context random variable is denoted by $C_i = X_{<i}$ and its outcome by $c_i = x_{j_{<i}}$. The chain rule formula then gives the full corpus likelihood $P^*(X_1, \ldots, X_N) = \Pi_{i=1}^N P^*(X_i | C_i)$. Therefore, we view natural language as a set of conditional next-token distributions $\mathcal{S} = \{(c_1, P^*(X|c_1)), \ldots, (c_N, P^*(X|c_N)) \}$. 
%
%
%We would like to learn an auto-regressive probabilistic model to predict the next word given a \textit{context} of past words, i.e. to model $P^*(X_i | X_{i-1}, \ldots, X_1) = P^*(X_i | C_i)$, where $C_i = X_{<i}$ and $P^*$ is the true probability distribution of our corpus. We make use of the standard chain rule formula for conditional probability which gives:
 
%- Eq. 3: please either specify or bound the variables, and be clear about what are the unitary distributions on the RHS.

\par The goal of language models is to approximate the true $P^*$ with a parametric distribution $Q_{\theta}$. Popular and state of the art methods~\citep{takase2018direct,zolna2017fraternal,yang2017breaking,merity2017regularizing,melis2017state,krause2017dynamic,merity2016pointer,grave2016improving} use recurrent neural networks (RNNs) such as stacked LSTMs~\citep{hochreiter1997long} to represent each context $c_i$ as a vector of fixed dimension $d$ denoted by $\h_i \in R^d$. Words are also embedded in the same continuous space, i.e. word $x_j$ is mapped to vector $\w_j \in \R^d$. Typically $d << M$. The RNN cell is a function that allows to express the context vectors recursively: $\h_{i+1} = RNN(\h_i, \w_i)$. Finally, the conditional probability of the next word in a context is given by the \textbf{Linear-Softmax} model which is a standard Softmax function on top of the word-context dot-product logits:
\begin{equation}
Q_{\Theta}(x_i | c_j) = \frac{ \exp(\h_j^{\top}\w_i)} {\sum_{s = 1}^M \exp(\h_j^{\top}\w_s)}
\label{eq:q}
\end{equation}
$\Theta$ being the model's parameters. Training is done by minimizing the cross entropy (or its exponential, the \textit{perplexity})
\begin{equation}
\mathcal{L}(\Theta) = \frac{1}{N} \sum_{i=1}^N -\log Q_{\Theta}(x_i | c_j)
\end{equation}
which is an approximation of the true expected cross entropy
\begin{equation}
\begin{split}
\mathcal{L}(\Theta) \approx \E_C \E_{X} [-\log Q(X | C)] = \E_C [H(P^*,Q | C)]
\end{split}
\end{equation}

Active LM research focuses on better context embedding models, optimization, long range dependencies or caching techniques. Inspired by~\citep{yang2017breaking}, we here focus on investigating and alleviating the \textit{bottleneck of the Linear-Softmax model}.

\section{Softmax Bottleneck - Problem and Insights}
\label{problem}

\paragraph{Main questions.} In the above model of \cref{eq:q} we made the assumption that any (conditional) probability distribution over a large word vocabulary $\mathcal{V}$ can be "well" parameterized by a single low-dimensional vector $\h$ and an exponential family distribution (Linear-Softmax), while also having access to a set of word embeddings $\{\w_i, i = 1, \ldots , M\}$ shared across all data distributions in the real set $\mathcal{S}$:
\begin{equation}
Q_{\h}(x_i) := Q_{\Theta}(x_i | c) = \frac{ \exp(\h^{\top}\w_i)} {\sum_{i' = 1}^M \exp(\h^{\top}\w_{i'})}
\label{eq:q_generic}
\end{equation}
One question we would like to theoretically and empirically investigate is: \\
\textit{Is one embedding vector $\h$ enough to fully represent any distribution of interest, i.e. can we always find $\Theta$ s.t. $Q_{\Theta}(X | c) = P^*(X|c)$ for all distributions of interest $P^*(\cdot | c) \in \mathcal{S}$ ? If not, how "close" can we get in terms of different interesting metrics (e.g. fitting the logits matrix, cross entropy, mode matching) ?}

We will see that Linear-Softmax is indeed limited. Next, to alleviate this bottleneck, we will introduce the Linear-Monotonic-Softmax (LMS) model and we will take steps in re-investigating the above question.

\paragraph{Connection with Matrix Factorization.} We follow the formalism of~\citep{yang2017breaking} and define the \textit{ log-P matrix} associated with any family of conditional probability distributions $P$ over all possible contexts:
\begin{equation}
\A_P \in \R^{M \times N}, \quad (A_P)_{ij} = \log P(x_i|c_j)
\end{equation}
We further define the context and word matrices:
\begin{equation}
\label{eq:W_H}
\HH_{\Theta} = \begin{bmatrix}
    \h_1^{\top} \\ \h_2^{\top} \\ \ldots \\ \h_N^{\top}
\end{bmatrix} \in \R^{N \times d}, \quad 
\W_{\Theta} = \begin{bmatrix}
    \w_1^{\top} \\ \w_2^{\top} \\ \ldots \\ \w_M^{\top}
\end{bmatrix} \in \R^{M \times d}
\end{equation}
as well as the logits matrix $\W_{\Theta} \HH_{\Theta}^{\top}$. Then, one derives that
\begin{equation}
\A_{Q_{\Theta}} = \W_{\Theta} \HH_{\Theta}^{\top} - \e_M \cdot {\bf logZ}^{\top}
\end{equation}
where $\e_M= \begin{bmatrix}
   1 \\1 \\ \ldots \\ 1
\end{bmatrix} \in \R^M$, and ${\bf logZ} = \begin{bmatrix}
   \log Z_1 \\ \log Z_2 \\ \ldots \\ \log Z_N
\end{bmatrix} \in \R^N$ is the vector of log-partition functions in \cref{eq:q}.

 Denoting by $r(\cdot)$ the matrix rank function, one has 
\begin{equation*}
r(\e_M \cdot {\bf logZ}^{\top}) = 1, \quad r(\W_{\Theta} \HH_{\Theta}^{\top}) \leq d, \quad r(\A_{Q_{\Theta}}) \leq d+1
\end{equation*} 
Where the rightmost inequality is proved using a classic rank inequality\footnote{$r(\B+\C) \leq r(\B) + r(\C), \forall \B,\C$ matrices of the same dimensions.}. Moreover, $r(\A_{Q_{\Theta}}) \geq d-1$ if $r(\W_{\Theta} \HH_{\Theta}^{\top})  = d$, which shows that the log-partition functions cannot change the final rank by more than 1. Since $\A_{P^*}$ is likely of full rank M for real distributions,~\citep{yang2017breaking} note that $\A_{P^*} \neq \A_{Q_{\Theta}}$ when $d < M - 1$, meaning that the \textit{Linear-Softmax} model has a representational bottleneck.

\paragraph{Quantifying the Error.} The above exposure shows one face of the coin, but, in practice, we might also be interested to know how "bad" this bottleneck can be. This depends on the choice of the distance function between discrete probability distributions. Such functions may be  explicitly minimized in order to learn the parametric distribution $Q_{\Theta}$ closest to the true (unknown) data distribution.

\paragraph{a) Mean Squared Error.} Assuming we remain in the matrix factorization setting, one natural choice of such a distance  is mean square error (MSE):
\begin{equation}
\mathcal{L}_{MSE}(\Theta) = \frac{1}{N} \|\A_{P^*} - \A_{Q_{\Theta}}\|_F^2
\end{equation}
A simple consequence of the Eckart-Young-Mirsky theorem is the following result:
\begin{theorem}
\label{thm:eckartyoung}
$\forall \Theta, \|\A_{P^*} - \A_{Q_{\Theta}}\|_F^2 \geq \sqrt{\sigma_{d+2}^2 + \ldots + \sigma_M^2}$
where $\sigma_1 \geq \sigma_2 \geq \ldots \geq \sigma_M$ are the singular values of matrix $\A_{P^*}$. 
\end{theorem}
Proof is in \cref{app:eckartyoung}. This result shows that, under the MSE distance, we cannot find a model arbitrary close to our true distribution if we have a rank deficiency on $\A_{Q_{\Theta}}$. In \cref{sec:mon_sol}, we will also investigate the rank deficiency for the Linear-Monotonic-Softmax (LMS) model.

\par However, the MSE error has a major drawback when used to quantify how well two distributions match: it puts emphasis on matching the tail of the distributions rather then their means or modes. To see this intuitively, we use the inequality:
\begin{equation}
\frac{1}{x + \epsilon} < \frac{\log(x+\epsilon) - \log(x) }{\epsilon} < \frac{1}{x}, \quad \forall x, \epsilon > 0
\end{equation}
which, since $\lim_{x \rightarrow 0} 1/x = \infty$, shows that mis-matching the small values of $\log P(x)$ incurs a much higher error compared to the large values. This behavior can be highly undesirable in practical settings such as prediction of the most likely next word or class, especially since a wide variety of real-world distributions exhibit a power law (e.g. Zipf's law for natural language~\citep{manning1999foundations}).

\paragraph{b) Cross Entropy.} The most used loss for discrete data is cross entropy, so it is natural to analyze the Softmax bottleneck in terms of its minimum value.

For a single (one context) true distribution $P^*(X)$ and a parametric model $Q_{\h}(x_i) \propto \exp(\inp{\w_i}{\h})$ with fixed word embeddings $\W$ and variable (learnable) context vector $\h$, this loss is:
\begin{equation*}
H(P^*,Q_{\h}) = \E_{P^*}[- \log Q_{\h}] = - \inp{\E_{P^*}[\w]}{\h} + \log Z^{(\h)}
\end{equation*}
where $\E_{P^*}[\w] = \sum_{j=1}^M P^*(x_j) \w_j \in \R^d	$ and the partition function is $Z^{(\h)} := \sum_{j=1}^M \exp(\inp{\w_j}{\h})$. 

A question is: \textit{What is the minimum achievable cross-entropy for a learnable vector $\h \in \R^d$ ?}

Towards this direction, we make the connection with the Maximum Entropy Principle under Linear Constraints via the following theorem.		
\begin{theorem}
\label{thm:maxentropy} Let $H(R) = - \sum_{i=1}^M R(x_i) \log R(x_i)$ be the entropy of the discrete distribution $R$. Then:

i) $H(P^*,Q) \geq H(P^*)$, for any probability distribution $Q$ (not necessarily from the Linear-Softmax/exponential family).

ii) $min_{\h \in \R^d} H(P^*,Q_{\h}) = max_{R \in \mathcal{P^*}} H(R)$, where \\
$\mathcal{P}^* := \{R |\ R \geq 0,\ \sum_{i=1}^M R(x_i) = 1,\ \E_{R}[\w] = \E_{P^*}[\w]\}$
is a convex polytope defined by d+1 linear constraints. 
\end{theorem}
A proof is given in \cref{app:maxentropy}. One can see that increasing the word embedding dimension and assuming the word embedding matrix W has full rank (i.e. the new constraints cannot be derived from the previous constraints) implies that the polytope $\mathcal{P}^*$ "shrinks", i.e. the maximum entropy becomes lower. Thus, the following hold.
\begin{corollary} The minimum achievable cross-entropy $H(P^*,Q_{\h})$ becomes lower as the word embedding dimension increases, if W keeps having full rank.
\end{corollary}
\begin{corollary}
If $d = M$ and the word embedding matrix W has full rank, then $\mathcal{P}^* = \{P^*\}$ and the lowest possible cross entropy is achieved: $min_{\h} H(P^*,Q_{\h}) = H(P^*)$.
\end{corollary}

\paragraph{c) Mode Matching.} In classification or generative models for discrete data we are often interested in predicting the modes of the true data distributions, e.g. the most likely next word given a context. Thus, we hope that a parametric model trained with our loss of choice (e.g. cross entropy) also exhibits a high accuracy at matching the modes. In our setting, this is represented by the success percentage:
\begin{equation}
\frac{1}{N} \#\{j: \arg\max_{i} P^*(x_i| c_j) = \arg\max_{i} Q_{\Theta}(x_i|c_j) \}
\end{equation}
We will empirically estimate this quantity for a synthetic experiment in \cref{experiments}.

\paragraph{Breaking the Bottleneck via Mixture of Softmaxes (MoS).} ~\citep{yang2017breaking} propose to use a MoS to alleviate this bottleneck. Concretely, they move from single point context embeddings to K embeddings as
\begin{equation}
Q_{\Theta}^{MoS}(x_i|c_j) = \sum_{k=1}^K \pi_{j,k} \frac{\exp(\g_{j,k}^{\top}\w_i)}{\sum_{s=1}^M \exp(\g_{j,k}^{\top}\w_s)}
\end{equation}
where $\pi_{j,k} = \frac{\exp(\vv_{k}^{\top}\h_j)}{\sum_{k'=1}^K \exp(\vv_{k'}^{\top}\h_j)}$ are mixture priors, and $\g_{j,k} = \tanh(\U_{k} \h_j)$ are the K embeddings representing the context j. Here, $\vv_k$ and $\U_k$ are the model parameters, shared across all contexts. 

While effective and achieving state-of-the-art LM perplexities, this model is several orders of magnitude more expensive than Linear-Softmax, having no theoretical guarantees to the best of our knowledge.

%% file: sol_mon.tex
\section{Monotonic Pointwise Functions}
\label{sec:mon_sol}

Our main contribution is to analyze and learn pointwise non-linearities $f(\cdot)$ that would alleviate the Softmax bottleneck. We are thus interested in the \textbf{Linear-Monotonic-Softmax (LMS)} layer defined as
\begin{equation}
Q(x_i) = \frac{ \exp(f(\h^{\top}\w_i))} {\sum_{s = 1}^M \exp(f(\h^{\top}\w_{s}))}
\label{eq:q_mon_f}
\end{equation}
This model draws inspiration from non-metric multidimensional scaling~\cite{kruskal1964multidimensional,kruskal1964nonmetric}. We desire to restrict to pointwise functions that have the following properties:
\begin{itemize}
\item \textit{non-linearity}: to break the Softmax bottleneck, i.e. to not limit the rank of $f(\W_{\Theta} \HH_{\Theta}^{\top})$ to $d$
\item \textit{increasing}: to preserve the ranking/order of logits
\item \textit{bijectivity on $\R$}: $\lim_{x \rightarrow \pm \infty} f(x) = \pm \infty$ to have no obvious limitation in modeling sparse or other distributions
\item \textit{continuous and (piecewise) differentiable}: to be learned using backpropagation
\item \textit{fast and memory efficient}: to add little overhead compared to Linear-Softmax and unlike MoS
\end{itemize}
We first note that our model is a generalization of vanilla linear Softmax, which can be recovered by taking the identity function in \cref{eq:q_mon_f}. Another particular example of a function with the above properties is $2x - \log(1 + \exp(x))$. This is the main focus of~\citep{kanai2018sigsoftmax}, but here we generalize their approach by investigating and learning generic parametric pointwise increasing functions.

We will show in \cref{thm:any_f_then_monotone} that the above first 4 conditions are not limiting the expressiveness of our models in terms of matrix rank deficiency. Moreover, in \cref{thm:plif} we show that the class of continuous piecewise linear increasing functions is a universal approximator for all differentiable increasing functions with bounded derivative that are defined on a finite interval. Related to the last property, we will explain why these functions are fast and memory efficient.

\textit{Notations:} For any matrix $\A \in \R^{M \times N}$ and pointwise function $f: \R \rightarrow \R$, we denote by $f(\A)$ the matrix $\B \in \R^{M \times N}$ with $B_{ij} = f(A_{ij})$. In the case $f(x) = x^p$, we will follow~\citep{amini2012low} and use the notation $\A^{\odot p}$.

We list our main theoretical results for pointwise functions. 

\paragraph{How powerful are monotonic pointwise non-linearities?}  We prove that the conditions imposed above on pointwise $f$'s are not restrictive when concerned about matrix rank increase.
\begin{theorem}
\label{thm:any_f_then_monotone} 
Let $\A\in \R^{M \times N}$ be any fixed real matrix of any rank. If there exists a pointwise function $f: \R \rightarrow \R$ s.t. $f(\A)$ has rank at least $K$, then there also exists a bijective, continuous, piecewise differentiable  and strictly increasing function $g: \R \rightarrow \R$ s.t. $g(\A)$ has rank at least $K$.
\end{theorem}
Proof is in \cref{app:any_f_then_monotone}.

\paragraph{Making a matrix full-rank via pointwise operators.} Theorem \ref{thm:any_f_then_monotone} shows that we only need to characterize low-rank matrices for which there exists any pointwise operator that increases its rank. In the most useful case, we would like to know when such operators can make it full rank. 
%\hi{mmmmm}We will show that this happens almost always. 
But, first, we observe that not all matrices can be made full-rank no matter what pointwise function one uses, for example matrices that have the same column repeated, or those that have two columns with constant entries. Next, we state a simple, but practically useful result for our language model formalism. Proof is in \cref{app:distinct_dot_products}.
\begin{lemma}
\label{thm:distinct_dot_products} 
Let $\A \in \R^{M \times N}, M \leq N$ be any fixed real matrix of rank at most $d$, i.e. $\A = \W \HH^{\top}$ where $\W \in \R^{M \times d}, \HH \in \R^{N \times d}$. Denote by $\h_i$ and $\w_i$ the i-th rows in $\HH$ and $\W$. If one can find $M$ distinct rows $j_1, \ldots, j_M$ in $\HH$ s.t. the values $\inp{\w_i}{\h_{j_i}}$ are distinct from all the other entries of matrix $\A$, then there exists a pointwise function $f: \R \rightarrow \R$ s.t. $f(\A)$ has full rank M.
\end{lemma}
%$W$ and $H$ are called the "word" and "context" matrices following the notation from Eq.~\ref{eq:W_H}.

Next, we focus on simple power operators $\A^{\odot p}$ and cite a previous result that shows a limitation: small $p$ values cannot make the matrix rank arbitrarily large.
\begin{theorem}~\citep{amini2012low}
Let $\A \in \R^{N \times M}$ be a rank $d$ matrix. Let $p$ be any positive integer. Then
\begin{equation}
r(\A^{\odot p}) \leq \min \left \{N, M, \binom{d + p - 1}{p} \right\}
\end{equation} 
\end{theorem}
However, $\lim_{p \rightarrow \infty} \binom{d + p - 1}{p} = \infty, \forall d > 1$, so there is still hope we can find monomials that make a matrix full rank if we look at sufficiently large powers. The following novel result  proved in \cref{app:sq_fullrank} confirms in a particular case that this is almost surely achieved. Let $O_k^N = \{ \A \in \R^{N \times N} : r(\A) = k \}$ be the submanifold of $\R^{N \times N}$ consisting of rank $k$ matrices.
  
\begin{theorem}
 
\label{thm:sq_fullrank}
 
For $N > 1$, the pointwise function $f(x)=x^2$ makes matrices in $O_{N-1}^N$ to almost surely become full rank.
 
\end{theorem}

%\hi{TODO: remove from here and from abstract ??}
% The following two novel results confirm this is almost surely achieved. Let $O_k^N = \{ A \in \R^{N \times N} : r(A) = k \}$ be the submanifold of $\R^{N \times N}$ consisting of rank k matrices.
%
%\begin{theorem}
%\label{thm:sq_fullrank}
%For $N > 1$, the pointwise function $f(x)=x^2$ makes matrices in $O_{N-1}^N$ to almost surely become full rank.
%\end{theorem}
%Proof is in Appendix~\ref{app:sq_fullrank}. This theorem is then used to prove:
%\begin{theorem}
%\label{thm:d_fullrank}
%For $d \geq 1$, the pointwise function $f(x)=x^{2^{N-d}}$ makes matrices in $O_{d}^N$ to almost surely become full rank.
%\end{theorem}
%Proof is in Appendix~\ref{app:d_fullrank}. 

\begin{figure}
\begin{center}
\includegraphics[scale=0.09]{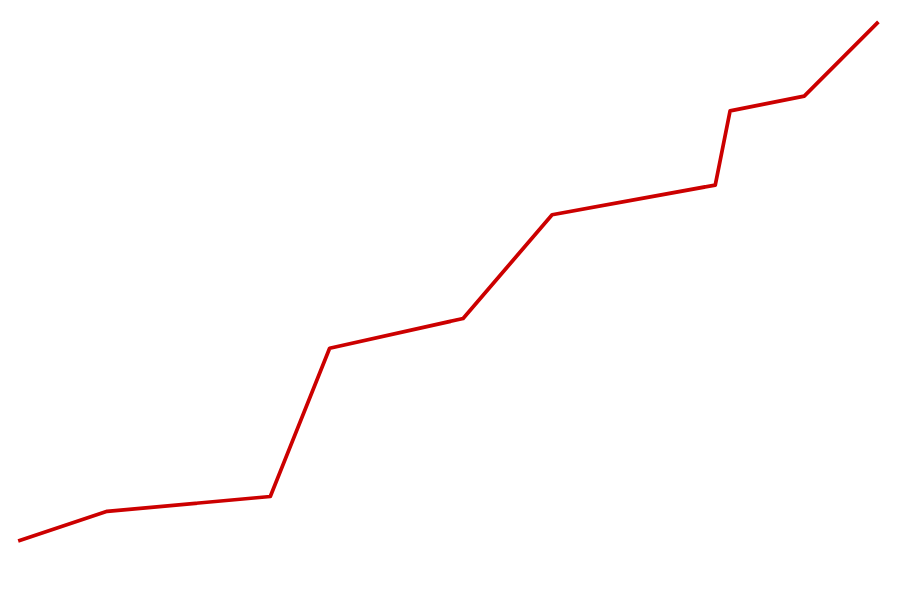}
\caption{Example of an increasing continuous piecewise linear function.}
\label{fig:mon}
\end{center}
\end{figure}

\paragraph{Architecture(s) of Learnable Monotonic Functions.} Even though some particular pointwise functions such as monomials/polynomials can make a low-rank matrix to be full-rank and, thus, remove the rank deficiency bottleneck, it is not guaranteed that this new matrix is close to the true data matrix. For this reason, we propose to learn parametric pointwise functions together with our model. For the reasons previously described, we design these functions to be bijective on full $\R$, increasing, continuous and (almost everywhere) differentiable. We note that the problem of learning parametric monotonic functions was analyzed by~\citep{sill1998monotonic}, but here we use the architecture proposed in~\citep{daniels2010monotone}, namely:
\begin{equation}
\label{eq:monotonic_model_K}
f(x) = \sum_{i=1}^K v_i \sigma(u_ix + b_i) + b 
\end{equation}
where $u_i,v_i,b_i,b \in \R, u_i,v_i \geq 0, \forall i \in \{1, \ldots , K\}$. This corresponds to an one hidden layer neural network with K hidden units and positively constraint weights (but not the biases).~\citep{daniels2010monotone} prove that this class of functions is universal approximator for all continuous increasing functions. However, in practice, one has to use a large number of hidden units $K$ in order to achieve a good approximation. While for our synthetic experiments in \cref{experiments} this was not an issue, for the real language modeling experiments this results in a heavy computational overhead. To understand why, in language modeling one has to process at a time large minibatches of contexts, meaning that matrices of size $N \times M$ \footnote{M is the vocabulary size, N is the number of contexts in a minibatch.} have to be stored in the GPU memory. If one uses the above architecture or the MoS architecture~\citep{yang2017breaking}, one has to store in the GPU memory and process intermediate tensors of dimension $N \times M \times K$, which can result in a significant running time and memory overhead even for small values of $K$ such as 10 or 15. This may lead to smaller batch sizes and thus impede the model's scalability.

To address the above computation problem, we propose to use an efficient class of parametric piecewise linear increasing functions called {\bf PLIF} = \textbf{P}iecewise \textbf{L}inear \textbf{I}ncreasing \textbf{F}unctions. An example is shown in \cref{fig:mon}. Formally, we fix a (large enough) interval $[-T,T]$ and K + 1 equally distanced knots in this interval: $l_i = -T + \frac{2Ti}{K}, \forall 0 \leq i \leq K$. We define our function to be piecewise linear, meaning that $f(x) = s_ix + b_i, \forall x \in [l_i, l_{i+1}]$, where $s_i$ is the slope of the linear function on the interval $[l_i, l_{i+1}]$. To enforce monotonicity, we need $s_i > 0$ which is achieved using the parametric form $s_i = \log(1 + \exp(v_i))$, where $v_i$ are unconstrained parameters. Moreover, we need the function to be continuous, meaning that $f$ has the same value in each knot $l_i$. This is achieved iff $\forall i > 0, b_i = b_0 + s_0l_0 - l_is_i + \frac{2T}{K} \sum_{j=0}^{i-1}s_j$. Thus, this model has as (learnable) parameters the values $s_i$ and the initial bias $b_0$. 

\paragraph{Computational Efficiency of PLIF.} The above formulation of the PLIF model is computationally efficient: i) a forward pass for an input $x$ is done by converting $x$ to the index $i_x : = \lfloor (x + T) \frac{K}{2T} \rfloor $, doing two lookups for index $i_x$ in the vectors $\textbf{s} := \{s_i:  0 \leq i \leq K\}$ and in $cumsum(\textbf{s})$, and then returning the value $s_{i_x}x + b_{i_x}$. ii) The backward pass only updates $s_{i_x}$ and $cumsum(\textbf{s})_{i_x}$, which have efficient implementations. Thus, the additional running time is negligible, while the additional memory only consists of two K-dimensional vectors and does not depend on the minibatch size like MoS or the model in \cref{eq:monotonic_model_K} do. For these reasons, we are computationally able to use large values of K (e.g. $10^5 - 10^6$) which offers great flexibility in modeling highly non-linear functions.

\paragraph{Universal Approximation Property of PLIF.} From a theoretical perspective, we state the following result and prove it in \cref{app:plif}:
\begin{theorem}
\label{thm:plif}
The PLIF model with large enough number of knots $K$ can approximate arbitrarily well any differentiable increasing real function defined on [-T,T] that has bounded derivatives.
\end{theorem}

%% file: experiments.tex
\section{Experiments}
\label{experiments}

We empirically assess Linear-Softmax, Linear-Monotonic-Softmax (LMS) and Mixture of Softmaxes (MoS).

%
%\begin{figure}
%\begin{center}
%\includegraphics[scale=0.08]{dir.png}
%\caption{Distribution of 3-class discrete distributions sampled from a Dirichlet prior with equal concentration parameters (from left to right: 1.3, 3, 7). Larger concentration parameters result in close to uniform distributions, while low values result in sparse distributions. Image source: Wikipedia.}
%\label{fig:dir}
%\end{center}
%\end{figure}

\begin{figure*}
\begin{center}
\includegraphics[scale=0.15]{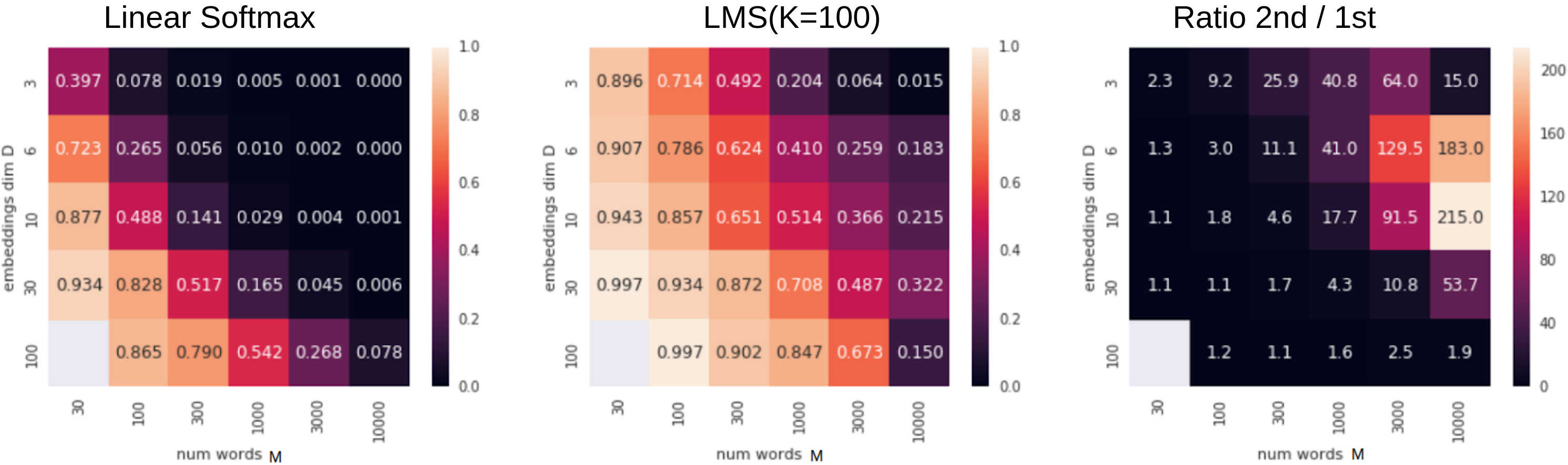}
\includegraphics[scale=0.155]{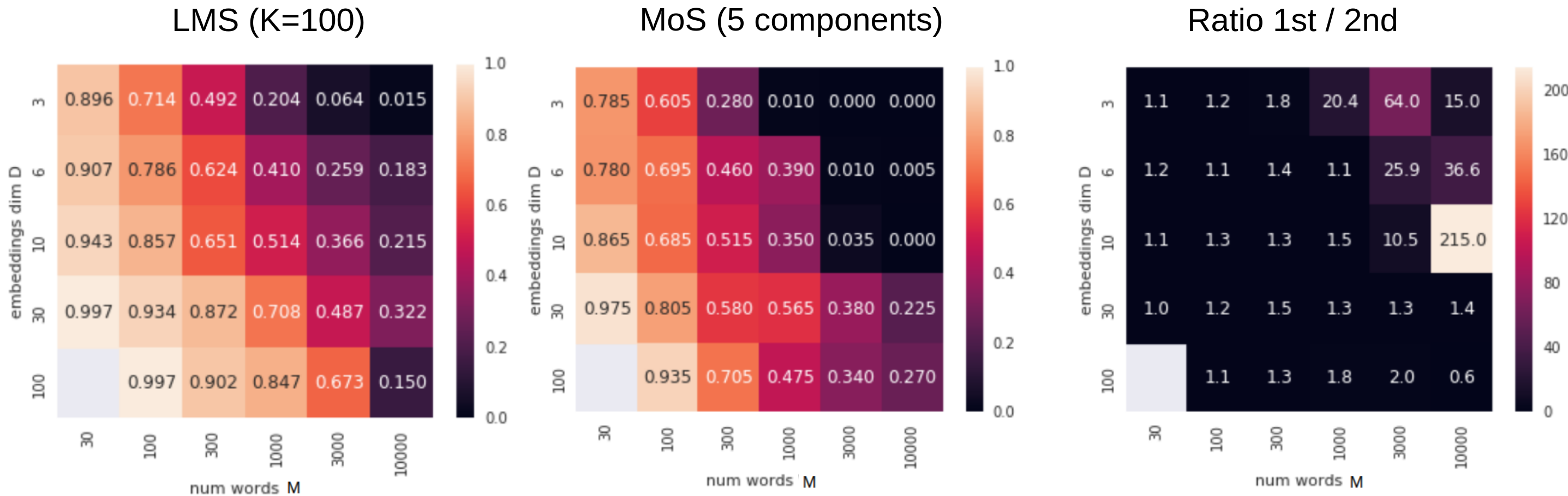}
\caption{Percentage of contexts $j$ for which the modes of true and parametric distributions match, i.e $\arg\max_{i} P^*(x_i|c_j) = \arg\max_{i} Q_{\Theta}(x_i|c_j)$. Higher the better. Dirichlet concentration $\alpha = 0.1$.}
\label{fig:argmax_0_1}
\end{center}
\end{figure*}

\begin{figure*}
\begin{center}
\includegraphics[scale=0.16]{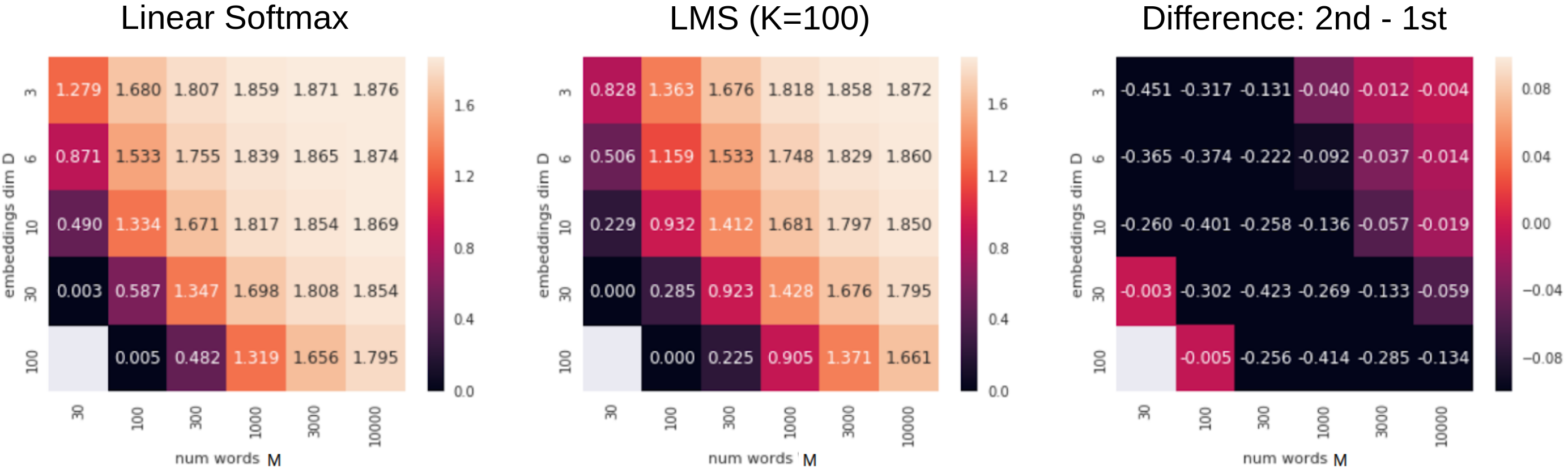}
\includegraphics[scale=0.165]{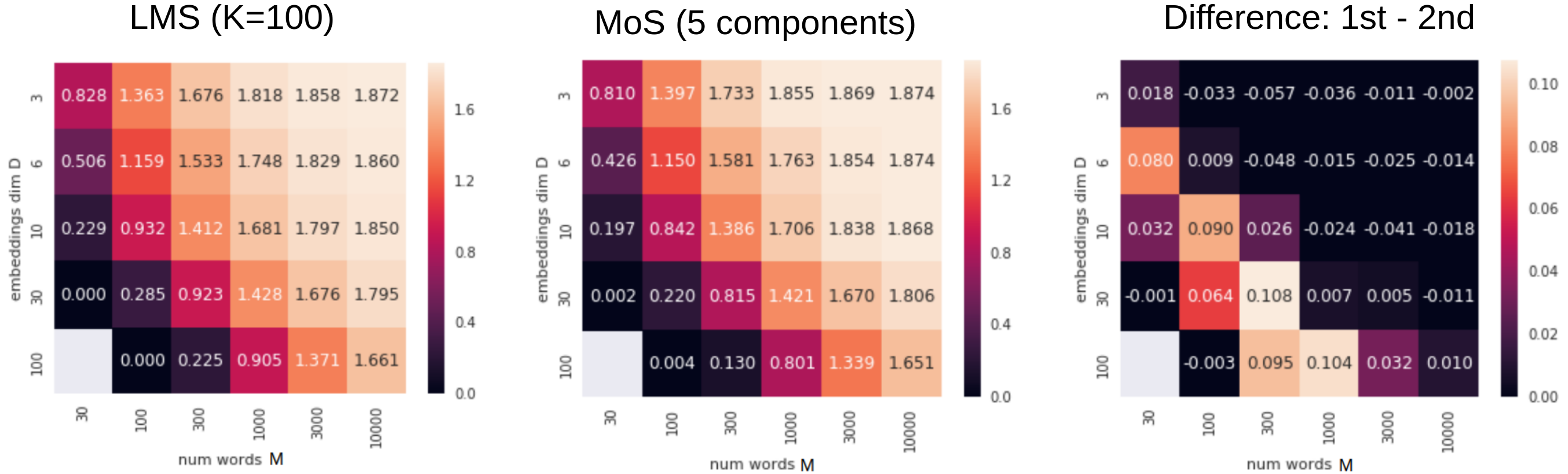}
\caption{Average $KL(P^* || Q_{\Theta})$ (across all contexts). Lower the better. Dirichlet concentration $\alpha = 0.1$. }
\label{fig:kl_0_1}
\end{center}
\end{figure*}

\begin{table*}[t]
\caption{Single model perplexities on validation and test sets on Penn Treebank and WikiText-2 datasets. For a fair comparison, baseline results are obtained by running the respective open-source implementations locally, however, being comparable to the published results. We also show the training time per epoch when using a single Tesla P100 GPU.}
%\vskip 0.15in
\begin{center}
\begin{scriptsize}
\begin{sc}
\begin{tabular}{| c || c | c | c | c || c | c | c | c | c |}
\toprule
                      &         \multicolumn{4}{c||}{Penn Treebank}  & \multicolumn{4}{c|}{WikiText-2} \\  \hline

                      &   \#Param  &  Valid ppl  &   Test ppl &   \#sec/ep &   \#Param  &  Valid ppl  &   Test ppl &   \#sec/ep \\
\midrule

\begin{tabular}{@{}c@{}}Linear-Softmax\\\tiny{w/} AWD-LSTM, \tiny{w/o finetune}\\~\citep{merity2017regularizing}\end{tabular}
         & 24.2M  & 60.83 & 58.37 & $\sim$60 & 33M & 68.11 & 65.22 & $\sim$120 \\ \hline

\begin{tabular}{@{}c@{}}Ours LMS-PLIF, $10^5$ knots\\\tiny{w/} AWD-LSTM, \tiny{w/o finetune}\end{tabular}
        & 24.4M  & 59.45 & 57.25 & $\sim$70 & 33.2M & 67.87 & 64.86 & $\sim$150 \\ \hline

\begin{tabular}{@{}c@{}}MoS, K = 15\\\tiny{w/} AWD-LSTM, \tiny{w/o finetune}\\~\citep{yang2017breaking}\end{tabular}
         & 26.6M  & 58.58 & 56.43 & $\sim$150 & 33M & 66.01 & 63.33 & $\sim$550 \\ \hline

\begin{tabular}{@{}c@{}}MoS(15 comp) +\\our PLIF ($10^6$ knots)\\\tiny{w/} AWD-LSTM, \tiny{w/o finetune}\end{tabular}
         & 28.6M  & 58.20 & 56.02 & $\sim$220 & - & - & - & - \\ \hline

%\bottomrule
\end{tabular}
\end{sc}
\end{scriptsize}
\end{center}
\vskip -0.1in
\label{tab:results}
\end{table*}

\subsection{Synthetic Experiments}
We first explore a synthetic experimental setting that has the following advantages:
\begin{itemize}
\item allows to separate the Softmax bottleneck from other  bottlenecks, e.g. in the RNN context embedding layer.
\item allows to understand how powerful these models are to represent very different distributions using single low dimensional vectors, i.e. we remove the dependency between context vectors that happens when embedding contexts with a shared neural network.
\item allows to evaluate how well the modes of the true and the parametric distributions match, a metric of practical importance (e.g. for text generative models) that can be quantified only when given access to the true data distribution.
\end{itemize} 

To this end, we repeatedly sample $N$ different "true" discrete distributions over a fixed synthetic word vocabulary of size $M$. We use a Dirichlet prior with all concentration parameters equal to $\alpha$:
\begin{equation}
P^*(\cdot|c_j) \sim \text{Dir}(\alpha), \ \text{for\ } j=1, \ldots, N
\end{equation}
Larger $\alpha$'s result in close to uniform distributions, while low values result in sparse distributions. 
The effect of $\alpha$ is also shown for different values of $M$ in \cref{fig:dir_conc} from \cref{app:dir_conc}.

We learn parametric models $Q_{\Theta}(\cdot|c_j)$ to match the true $P^*$ distributions. We learn a set of word embeddings shared across all contexts and a separate context embedding per each distribution $Q_{\Theta}(\cdot|c_j)$. All embeddings have dimension $D$. We use the Linear-Softmax model as defined by \cref{eq:q}, the Mixture of Softmaxes (MoS) model~\citep{yang2017breaking}, and our LMS model given by \cref{eq:q_mon_f} with a pointwise monotonic function parameterized using the K hidden units architecture shown in \cref{eq:monotonic_model_K}. Learning the models' parameters is done by minimizing the cross entropy which is equivalent to minimizing the divergence $KL(P^* || Q_{\Theta})$ for each context $c_j$.

\paragraph{Results.} We present the results for different Dirichlet parameter $\alpha$, vocabulary sizes $M$, embedding sizes $D$ and evaluation metrics (mode matching and cross entropy / KL divergence) in \cref{fig:argmax_0_1,fig:kl_0_1}, but also show additional results in \cref{app:synthetic} in \cref{fig:argmax_0_01,fig:kl_0_01,fig:argmax_1}. In all the settings, we used $N=10^5$ contexts, where N is the number of different distributions $P^*(\cdot|c_j)$. 

\paragraph{Discussion.} We observe that, in most of the presented cases, LMS outperforms Linear-Softmax and MoS on both the task of mode matching and on the minimum achievable cross-entropy (KL divergence). Especially in the "low D - large M" setting, the difference is significantly large showcasing the existence of the Softmax bottleneck and the merits of our LMS model.  

However, we note that there is still room for future work and improvements, for example mode matching still largely suffers for low D and large M.

\subsection{Language Model Experiments}

We move to the real setting of language modeling. Here, due to computational reasons discussed in \cref{sec:mon_sol}, we will use our PLIF architecture introduced in the same section.

\paragraph{Datasets.} Following previous work \citep{mikolov2012statistical,inan2016tying,kim2016character,zoph2016neural}, we use the two most popular LM datasets: Penn TreeBank \citep{mikolov2010recurrent} and WikiText-2 \citep{merity2016pointer}. These datasets have word vocabulary sizes of 10,000 and 33,000.%, respectively.

%\	{-0.5cm}
\paragraph{Baselines.} We integrate our PLIF layer on top of the state of the art language models of AWD-LSTM~\citep{merity2017regularizing} and AWD-LSTM+MoS~\citep{yang2017breaking}. Additionally, our PLIF architecture can also be combined with MoS instead of standard Softmax. We call this model "MoS + PLIF". We use the AWD-LSTM open source implementation \footnote{\url{http://github.com/salesforce/awd-lstm-lm}}. All the models in \cref{tab:results} \footnote{Except MoS on WT2 which took too long to run on a single GPU, thus reporting the published results.} were ran locally and we report these results; we did this to understand how different Softmax models compare with each other when using the exact same context embedding architecture. We note that~\citep{yang2017breaking} redo architecture search after integrating their MoS model, their goal being to reduce the number of parameters to the same size as AWD-LSTM. 

All the models are trained without finetuning~\citep{merity2017regularizing}. We use embedding dimension 400 for all the models in \cref{tab:results}. For optimization, we use the strategy described in~\citep{merity2017regularizing} consisting of running stochastic gradient descent (SGD) with constant learning rate (20.0) until the cross entropy loss starts stabilizing, and then switching to averaged SGD. This strategy was shown to improve state of the art language models~\citep{takase2018direct} and to consistently and by a large margin outperform popular adaptive methods such as ADAM~\citep{kingma2014adam}.

We did not include the Sigsoftmax baseline model as no significant improvement over Linear-Softmax was seen, neither locally nor in the original paper~\citep{kanai2018sigsoftmax}  (the w/o finetune setting). We note that this method is a particular case of our LMS model.

\paragraph{Results and Discussion.} Table~\ref{tab:results} shows the results. Our LMS-PLIF layer consistently improves over Linear-Softmax when combined with the same state-of-the-art AWD-LSTM context embedding architecture. The  computational prices (memory and training time) we pay for using LMS-PLIF are negligible compared to Linear-Softmax. However, while MoS outperforms our simple LMS-PLIF model, it is computationally several orders of magnitude more expensive, which is a practical advantage of our method. Finally, combining MoS and our PLIF model gives the best Penn TreeBank result, outperforming all baselines (but at the highest computational cost).

We show in \cref{tab:example_f} statistics of the slope values of a learned PLIF function, revealing its highly non-linear nature.

\begin{figure}
\begin{center}
\includegraphics[scale=0.2]{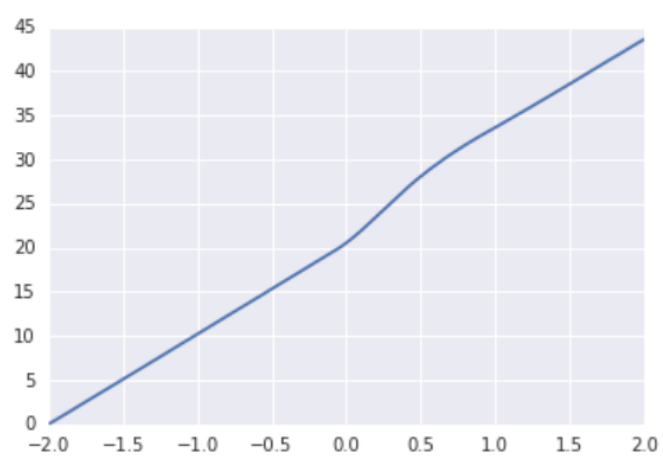}
\caption{Learned function for the model in \cref{tab:example_f}.}
\label{fig:plif}
\end{center}
\end{figure}

\begin{table}[t]
\caption{Statistics on the slope values of the PLIF pointwise function trained on WikiText-2.}
%\vskip 0.15in
\begin{center}
\begin{scriptsize}
\begin{sc}
\begin{tabular}{| c | c | c | c |}
\toprule
Mean  &  Std  &   Min &   Max \\
\midrule
1.10 & 0.62 & 0.02 & 5.16 \\ \hline
\end{tabular}
\end{sc}
\end{scriptsize}
\end{center}
\vskip -0.1in
\label{tab:example_f}
\end{table}

%% file: conclusion.tex
\section{Conclusion}
\label{conclusion}
We re-analyzed the Softmax bottleneck here from multiple perspectives and confirmed, both theoretically and empirically, that the widely used Softmax layer is not flexible enough to model arbitrarily distributions over large vocabularies. We proposed LMS-PLIF, a model that learns parametric monotonic functions to make Softmax more flexible, and show some of its capabilities. %There are several exciting future directions, including better theoretical understanding of the representational power of the LMS-PLIF layer and, empirically, exploring other supervised and unsupervised discrete problems with large number of classes.
%
%- future work: 
%-- lower bounds for x-entropy for monotonic fs
%-- analyze the effect of monotonic fs on the singular values of the input matrices and generalize eckart-young theorem. 
%-- empirically: analyze monotonic f in more settings, i.e. for supervised and unsupervised learning (BERT)
%-- apply for word embeddings and rec systems

%% file: appendix.tex
\clearpage

\section{Proof of Theorem~\ref{thm:eckartyoung}}\label{app:eckartyoung}
\begin{proof}
We derived in the main text that $r(\A_{Q_{\Theta}}) \leq d+1$. In addition, Eckart-Young-Mirsky theorem gives:
\begin{equation*}
\|\A_{P^*} - \B\|_F^2 \geq \sqrt{\sigma_{d+2}^2 + \ldots + \sigma_M^2} ,
\end{equation*}
\begin{equation*}
\forall \B \in \R^{M \times N}\ s.t.\ r(\B) \leq d+1
\end{equation*}
Thus, our result follows for $\B=\A_{Q_{\Theta}}$.
\end{proof}

%%%%%%%%%%%%%%%%%%%%%%%%%%%%%%%%%%%%%%%%%%%%%%%%%%%%%%%%%%%%%%%%%%%%%%%%%%%%%%%

\section{Proof of Theorem~\ref{thm:maxentropy}}\label{app:maxentropy}
\begin{proof}
i) Using the non-negativity property of the KL divergence, one derives:
\begin{equation}
KL(R\| Q) = H(R,Q) - H(R) \geq 0
\label{eq:klnonneg}
\end{equation}
for any probability distribution $R$. The result follows easily by taking $R = P^*$.\\

ii) Let $Q_{\h}(x_i) \propto \exp(\inp{\w_i}{\h})$. Then, for any probability distribution $R$, it is straightforward to derive that
\begin{equation}
H(R,Q_{\h}) = - \inp{\E_{R}[\w]}{\h} + \log Z^{(\h)}
\label{eq:log_cross_ent}
\end{equation}
Moreover, if $R \in \mathcal{P^*}$ is any distribution satisfying the d-dimensional linear constraints,  one derives from \cref{eq:log_cross_ent} that
\begin{equation}
H(P^*,Q_{\h}) = H(R,Q_{\h}),\ \forall R \in \mathcal{P^*}
\label{eq:equal_x_ents}
\end{equation}
combining \cref{eq:equal_x_ents,eq:klnonneg}, we get:
\begin{equation}
H(P^*,Q_{\h}) \geq H(R), \ \forall R \in \mathcal{P^*}
\end{equation}
thus 
\begin{equation}
H(P^*,Q_{\h}) \geq max_{R \in \mathcal{P^*}} H(R)
\end{equation}
which, since $Q_{\h}$ is arbitrary in the above exponential family, implies that
\begin{equation}
min_{\h} H(P^*,Q_{\h}) \geq max_{R \in \mathcal{P^*}} H(R)
\end{equation}

We are only left with proving the reverse, namely that $min_{\h} H(P^*,Q_{\h}) \leq max_{R \in \mathcal{P^*}} H(R)$. We use the standard derivations for the Maximum Entropy Principle, namely we form the Lagrangian:
\begin{equation}
\begin{split}
L(\bm{\lambda}, \beta, \h) := H(R) + & \beta \left(\sum_{i=1}^M R(x_i) - 1\right) + \\
 & + \inp{\bm{\lambda}}{\E_{R}[\w] - \E_{P^*}[\w]} 
\end{split}
\end{equation}
Setting its derivatives to 0, one gets that the optimal $R^* = \arg\max_{R \in \mathcal{P^*}} H(R)$ has the form
\begin{equation}
R^*(x_i) \propto \exp(\inp{\w_i}{\bm{\lambda}^*})
\end{equation}
for some $\bm{\lambda}^* \in \R^d$ that is chosen by solving the d-linear system $\E_{R^*}[\w] - \E_{P^*}[\w] = 0$. One can observe that $Q_{\bm{\lambda}^*} = R^*$, getting
\begin{equation*}
min_{\h} H(P^*,Q_{\h}) \leq H(P^*, Q_{\bm{\lambda}^*}) = H(P^*, R^*)
\end{equation*}
Finally, using \cref{eq:equal_x_ents}, we get:
\begin{equation*}
H(P^*, R^*) = H(R^*, R^*) = H(R^*) = max_{R \in \mathcal{P^*}} H(R)
\end{equation*}
which concludes the proof.
\end{proof}

%%%%%%%%%%%%%%%%%%%%%%%%%%%%%%%%%%%%%%%%%%%%%%%%%%%%%%%%%%%%%%%%%%%%
\section{Proof of Theorem~\ref{thm:any_f_then_monotone}}\label{app:any_f_then_monotone}
\begin{proof}
Since $f(\A)$ has rank at least $K$, there exists at least one submatrix $\M \in \R^{K \times K}$ of $\A$ such that $\det(f(\M)) \neq 0$. Let $b_1 < b_2 < \ldots < b_T$ be all the distinct values of $\M$. Denote by $\epsilon = \frac{1}{4} \min_{i > 1} |b_i - b_{i-1}|$. We first prove the following lemmas.

\begin{lemma}
\label{lemma:poly_zero}
Let $P \in \R[X_1, \ldots, X_T]$ be a multivariate polynomial with real coefficients. Assume there exist infinite sets $S_1, \ldots , S_T$ such that P vanishes on all the points of $S_1 \times S_2 \times \ldots \times S_T$. Then P vanishes on any point of $\R^T$.
\end{lemma}
\begin{proof}
We prove this by induction over $T$. The result easily holds for $T = 1$ since a real univariate non-zero polynomial can only have a finite set of roots. Assume now that the result holds for any polynomial in $T-1$ variables. We can write $P(X_1, X_2, \ldots , X_T)$ as a univariate polynomial in $X_1$ with coefficients polynomials in $\R[X_2, \ldots, X_T]$ as follows: $P(X_1, X_2, \ldots , X_T) = \sum_{i=0}^{d_1} Q_i(X_2, \ldots , X_T) X_1^i$, where $d_1$ is the maximum degree of $X_1$. For any arbitrary $x_2, \ldots , x_T \in S_2 \times \ldots \times S_T$, we know from the hypothesis that $P(c, x_2, \ldots , x_T) = 0, \forall c \in S_1$. Since $S_1$ is infinite we have that the univariate polynomial in $X_1$ is identical 0, i.e. $P(X, x_2, \ldots , x_T) \equiv 0$, which implies that $Q_i(x_2, \ldots , x_T) = 0$. However, $x_2, \ldots , x_T \in S_2 \times \ldots \times S_T$ were chosen arbitrarily, thus $Q_i(x_2, \ldots , x_T) = 0, \forall x_2, \ldots , x_T \in S_2 \times \ldots \times S_T$. Applying the induction hypothesis for $T - 1$, one gets that all $Q_i$ vanish on the full $\R^{T-1}$. Thus, $P(X, x_2, \ldots , x_T) \equiv 0, \forall (x_2 , \ldots, x_T) \in \R^{T-1}$, which implies that $P(x_1, x_2, \ldots , x_T) = 0, \forall (x_1, x_2 , \ldots, x_T) \in \R^T$.
\end{proof}

\begin{lemma}
\label{lemma:existence_c_i}
There exist $c_i \in [b_i - \epsilon, b_i + \epsilon], \forall i \in \{1, \ldots, T \}$ s.t. given any pointwise function $h$ satisfying $h(b_i) = c_i, \forall 1 \leq i \leq T$, we have $\det(h(\M)) \neq 0$.
\end{lemma}
\begin{proof}
Assume the contrary, that $\forall c_i \in [b_i - \epsilon, b_i + \epsilon]$, $\det(h(\M)) = 0$. 

We note that, using the  Leibniz formula of the determinant, one easily sees that $\det(\M)$ can be written as $P(b_1, \ldots, b_T)$, where $P \in \R[X_1, \ldots, X_T]$ is a multivariate polynomial in T variables. It is then easy to see that any pointwise $h$ will change the determinant of M as: $\det(h(\M)) = P(h(b_1), \ldots, h(b_T))$. Then, assuming this lemma is not true is equivalent with $P(c_1, \ldots, c_T) = 0, \forall c_i \in [b_i - \epsilon, b_i + \epsilon], \forall 1 \leq i \leq T$. Applying \cref{lemma:poly_zero} to sets $S_i = [b_i - \epsilon, b_i + \epsilon]$, one gets that $P(c_1, \ldots, c_T) = 0, \forall c_i \in \R, \forall i \in \{1, \ldots, T \}$. Taking $c_i = f(b_i)$ one obtains $\det(f(\M)) = P(f(b_1), \ldots, f(b_T)) = 0$ which is a contradiction with our assumption on $\M$ and $f$.
\end{proof}

We now return to the proof of the main theorem. For each $i \in \{1, \ldots, T \}$, let us denote by $c_i \in [b_i - \epsilon, b_i + \epsilon]$ the values from \cref{lemma:existence_c_i} that guarantee a non-zero determinant. We construct a pointwise bijective, piecewise differentiable, continuous and strictly increasing function $g : \R \rightarrow \R$ such that $g(b_i) = c_i$. It is obvious that $\det(g(\M))$ depends only on the values $g(b_i)$, so we are free to assign any other values to any other real input of $g$ as long as the above constraints on $g$ are satisfied. One example of such $g$ is a piecewise linear function defined  to match the following values: $g(b_i) = c_i, g(b_i + 2\epsilon) = b_i + 2\epsilon, \forall 1 \leq i \leq T, g(x) = x, \forall x < b_1 - 2\epsilon$ and  $g(x) = x, \forall x > b_T + 2\epsilon$. It can be easily seen that such a function is bijective, piecewise differentiable, continuous and strictly increasing.
\end{proof}

%%%%%%%%%%%%%%%%%%%%%%%%%%%%%%%%%%%%%%%%%%%%%%%%%%%%%%%%%%%%%%%%%%%%%%%%%%

\section{Proof of Lemma~\ref{thm:distinct_dot_products}}\label{app:distinct_dot_products}
\begin{proof}
If $\inp{\w_i}{\h_{j_i}}$ are distinct from all the other entries in the matrix A, one can design the following pointwise function:

\begin{equation*}
f(x)=
    \begin{cases}
      1 & \text{if}\ \exists i\ \text{s.t.\ } x = \inp{\w_i}{\h_{j_i}} \\
      0 & \text{else}
    \end{cases}
\end{equation*}

Then, let $\B$ be the $M \times M$ submatrix of A consisting of all its M rows and the M columns indexed by $j_i$'s. It is then clear that $f(\B) = \I_M$, which is obviously full rank.
\end{proof}

%%%%%%%%%%%%%%%%%%%%%%%%%%%%%%%%%%%%%%%%%%%%%%%%%%%%%%%%%%%%%%%%%%%%%%%%%%

\section{Proof of Theorem~\ref{thm:sq_fullrank}}\label{app:sq_fullrank}
 
\begin{proof}

We will make use of the following folklore lemmas:
 
\begin{lemma}
 
\label{lemma:finite_union}
Let $\mathcal{M}=\cup_i M_i$ be a finite union of Riemannian manifolds of dimension $m$, embedded in $\R^k$, with Riemannian metric $g_i$ inherited from $\R^k$. Then, any finite union $S$ of submanifolds of the $\mathcal{M}_i$'s of dimensions strictly smaller than $m$ is a set of null measure\footnote{w.r.t. the volume form of the manifold, \textit{i.e.} locally w.r.t. to the $m$-dimensional Lebesgue measure.}. In other words, any point from $\mathcal{M}$ is almost surely not in $S$.
\end{lemma}
 
\begin{proof}
 
(sketch) any submanifold of $\mathcal{M}$ of strictly smaller dimension than $m$ has volume or measure zero. The result then follows from the fact that a finite union of sets of measure zero has also measure zero.
 
\end{proof}

\begin{lemma}
 
\label{lemma:rank_d_manifold}
 
The set $O^N_k$ of rank-$k$ matrices of size $N\times N$ with $0 < k < N$ is a Riemannian manifold of dimension $2kN - k^2$ embedded in $\R^{N \times N}$.
 
\end{lemma}
 
\begin{proof}
 
See e.g.~\cite{shalit2012online}. The Riemannian metric for embedded manifolds is simply the Euclidean metric restricted to the manifold.
 
\end{proof}

We now return to the main proof of the theorem. From \cref{lemma:rank_d_manifold} we have that $\dim(O_{N-1}^N) = N^2 - 1$. We want to prove that the subset of $O_{N-1}^N$ of rank $N-1$ matrices for which $x^2$ is not increasing their rank has dimension strictly smaller than $\dim(O_{N-1}^N)$. In this case, using \cref{lemma:finite_union}, the measure of all ill-behaved matrices would be 0, so any matrix from $O_{N-1}^N$ is almost surely well-behaved, i.e. the rank of $\A^{\odot 2}$ is almost surely full rank $N$ for $\A \in O_{N-1}^N$.

We begin by removing from $O_{N-1}^N$ the set of all matrices that have two proportional columns, a set that we name $\Xi^N$. This is a finite\footnote{More precisely, of $\frac{N(N-1)}{2}$ manifolds, one per each pair of columns.} union of manifolds of dimension $N(N-1)+1$, namely all sets of matrices for which column i is proportional to column j, for all $1 \leq i < j \leq N $ \footnote{The $N(N-1)+1$ dimension comes from the fact that there are N-1 independent columns, plus a scalar, namely the multiplication factor between column i and column j.}. Using \cref{lemma:finite_union}, we derive that the measure or volume of $\Xi^N$ is 0.

Now, for any arbitrary $\A \in O_{N-1}^N \setminus \Xi^N$ with columns $\x^{(1)} , \ldots, \x^{(N)} \in \R^N$, one can easily derive that $\exists \gamma_i \in \R$ not all equal to 0 s.t. $\sum_{i=1}^N \gamma_i \x^{(i)} = 0$. We know that at least one $\gamma_i \neq 0$ from the fact that $\A \in O_{N-1}^N$; let us denote by $\Gamma^i$ the set of such matrices $\A \in O_{N-1}^N$. Since $O_{N-1}^N$ is the (finite) union of the $\Gamma^i$'s, we want to show that the set of ill-behaved matrices in each $\Gamma^i$ is contained in a manifold of dimension strictly smaller than that of $O_{N-1}^N$, which will conclude, using the fact that a finite union of null measure sets has null measure. 

Without loss of generality, let us assume that $\A\in \Gamma^N$, \textit{i.e.} that $\gamma_N\neq 0$. Let us note that 
\begin{equation}
\Gamma^N = \{\A \in O_{N-1}^N : \gamma_N  = 1 \},
\end{equation}
by substituting each $\gamma_i$ with $\gamma_i/\gamma_N$ for $1\leq i\leq N-1$.

If $\A^{\odot 2}$ is not full rank, there exist $\alpha_1,...,\alpha_{N}\in\R$ such that 
\begin{equation}\label{eq:blabla}
\sum_{i=1}^{N-1} \alpha_i (\x^{(i)})^{\odot 2} = \alpha_N\left(\sum_{i=1}^{N-1} \gamma_i \x^{(i)}\right)^{\odot 2}.
\end{equation}

For fixed $\alpha_1,...,\alpha_{N}\in\R$, denote by $M_\alpha$ the subset of the solutions $\{\x^{(i)}\}_{1\leq i\leq N-1}\subset\R^N$ of the above equation.

Define 
\begin{multline}
\varphi: (x^{(1)}_k,...,x^{(N-1)}_k) \in\R^{N-1} \mapsto \\
 \sum_{i=1}^{N-1} \alpha_i (x^{(i)}_k)^{ 2} - \alpha_N\left(\sum_{i=1}^{N-1} \gamma_i x^{(i)}_k\right)^{2}.
\end{multline}

This can be re-written $\varphi(\x)=\x^T \G\x$ with $$G_{ij}=\delta_{ij}(\alpha_i-\alpha_N\gamma_i^2) -(1-\delta_{ij})\alpha_N\gamma_i\gamma_j$$

It can be easily shown that since $\A$ is not in $\Xi^N$, $\G$ is not the null matrix. Indeed, if $\G=\mathbf{0}$, then either $\alpha_N=0$ $-$ and then $\alpha_i=\alpha_N\gamma_i^2=0$ for all $i$, which is excluded $-$ or $\alpha_N\neq 0$, and then $\alpha_N\gamma_i\gamma_j=0$ for all $i\neq j$, meaning only one $\gamma_{i_0}$ is non-zero, \textit{i.e.} $\x^{(N)} = -\gamma_{i_0} \x^{(i_0)}$ and hence $\A\in \Xi^N$.

Note that since $\G$ is not the null matrix, $dim(\ker \G)<N-1$. Furthermore, let $U:= \R^{N-1}\setminus \ker \G$. Invoking the Pre-Image theorem, the set $U\cap \varphi^{-1}(\{0\})$ is a submanifold of $\R^{N-1}$ of dimension $(N-1)-1=N-2$. Therefore, $\varphi^{-1}(\{0\})$ is a finite union of manifolds of dimensions smaller than (or equal to) $N-2$.

Since \cref{eq:blabla} can be written as an intersection of $N$ equations as the one defined by $\varphi$ (\textit{i.e.} one per coordinate), the set $M_\alpha$ of solutions of \cref{eq:blabla} is included in a finite union of manifolds of dimensions smaller than (or equal to) $N(N-2)$.

Finally, the total set $X$ of matrices we are after $-$ \textit{i.e.} of rank $N-1$ and which cannot be made full ranked by pointwise square $-$ can be defined as the union over $\alpha$ of all $M_\alpha$, \textit{i.e.} $X=\cup_\alpha M_\alpha$. As $X$ has the structure of a fiber bundle, with base space the set of $\alpha$'s (of dimension $N$), $X$ is a subset of submanifolds of dimensions smaller than $N + N(N-2) = N^2 -N < N^2 -1$ for $N>1$, which concludes the proof.

%Then, $\Gamma^N$ is a manifold of dimension $N^2 - 1$ because $- x^{(N)} = \sum_{i=1}^{N-1} \gamma_i x^{(i)}$, so all $\gamma_i$ and $x^{(i)}$ for $1 \leq i \leq N-1$ are independent, thus easily defining a chart from $\R^{N^2 - 1}$ to $\Gamma^N$. $\Gamma^N$ is also a Riemannian manifold since it inherits its metric from $O_{N-1}^N$.
% 

%In order to prove the main statement, it is enough to show that $\{A^{\odot 2} : A \in \Gamma^N\} \cap \{ A \in \R^{N \times N} : r(A) < N \}$ is a finite union of manifolds of dimension strictly smaller than $N^2 - 1$. 
% 

\end{proof}

%%%%%%%%%%%%%%%%%%%%%%%%%%%%%%%%%%%%%%%%%%%%%%%%%%%%%%%%%%%%%%%%%%%%%%%%%%

\section{Proof of Theorem~\ref{thm:plif}}\label{app:plif}
\begin{proof}
Let $h:[-T,T]$ be any increasing function defined on $[-T, T]$. Assume bounded derivatives, i.e. $\exists R > 0$ s.t. $|h'(x)| < R, \forall x \in [-T,T]$. Then, for a fixed positive integer K, we consider the knots $l_i = -T + \frac{2Ti}{K}, \forall 0 \leq i \leq K$. Next, using standard linear interpolation, we define a piecewise linear function $f_K: [-T, T] \rightarrow \R$ s.t. $f_K(l_i) = h(l_i), \forall 0 \leq i \leq K$. Since $h$ is increasing, one obtains that $f_K$ is also increasing. It is then easy to see that $f_K$ is a PLIF function. Moreover, the slopes are given by the formula: $s_i = \frac{h(l_{i+1}) - h(l_{i})}{l_{i+1} - l_i}$.

We define the additional function $g_K(x):= f_K(x) - h(x)$. We wish to prove that $\lim_{K \rightarrow \infty} \max_{x \in [-T, T]} |g_K(x)| = 0$ . For this, we first use Cauchy's theorem deriving that $\exists c_i \in (l_{i+1}, l_i)$ s.t. $s_i = \frac{h(l_{i+1}) - h(l_{i})}{l_{i+1} - l_i} = h'(c_i)$. Thus, since $h'$ is bounded by R, we get that $|s_i| < R, \forall i$. This further implies that $|g_K'(x)| < 2R, \forall x \in [-T,T]$. Moreover, from the definition of $f_K$ we have that $g_K(l_i) = 0, \forall i$. Finally, for any $x \in [-T,T]$, let $[l_{i+1}, l_i]$ be the interval in which $x$ lies. We have that:
\begin{equation}
\begin{split}
|g_K(x)| = |g_K(x) - g_K(l_i)| = \\
= \frac{|g_K(x) - g_K(l_i)|}{|x - l_i|} |x - l_i| \leq \\
\leq 2R|x - l_i| \leq 2R \frac{2T}{K}
\end{split}
\end{equation}
where the first inequality happens from the same argument derived from Cauchy's theorem as above. It is now trivial to prove that $\lim_{K \rightarrow \infty} \max_{x \in [-T, T]} |g_K(x)| = 0$, which concludes our proof.

\end{proof}

%%%%%%%%%%%%%%%%%%%%%%%%%%%%%%%%%%%%%%%%%%%%%%%%%%%%%%%%%%%%%%%%%%%%%%%%%%

\section{Effect of the Dirichlet concentration }\label{app:dir_conc}
See \cref{fig:dir_conc}.
\begin{figure*}
\begin{center}
\includegraphics[scale=0.23]{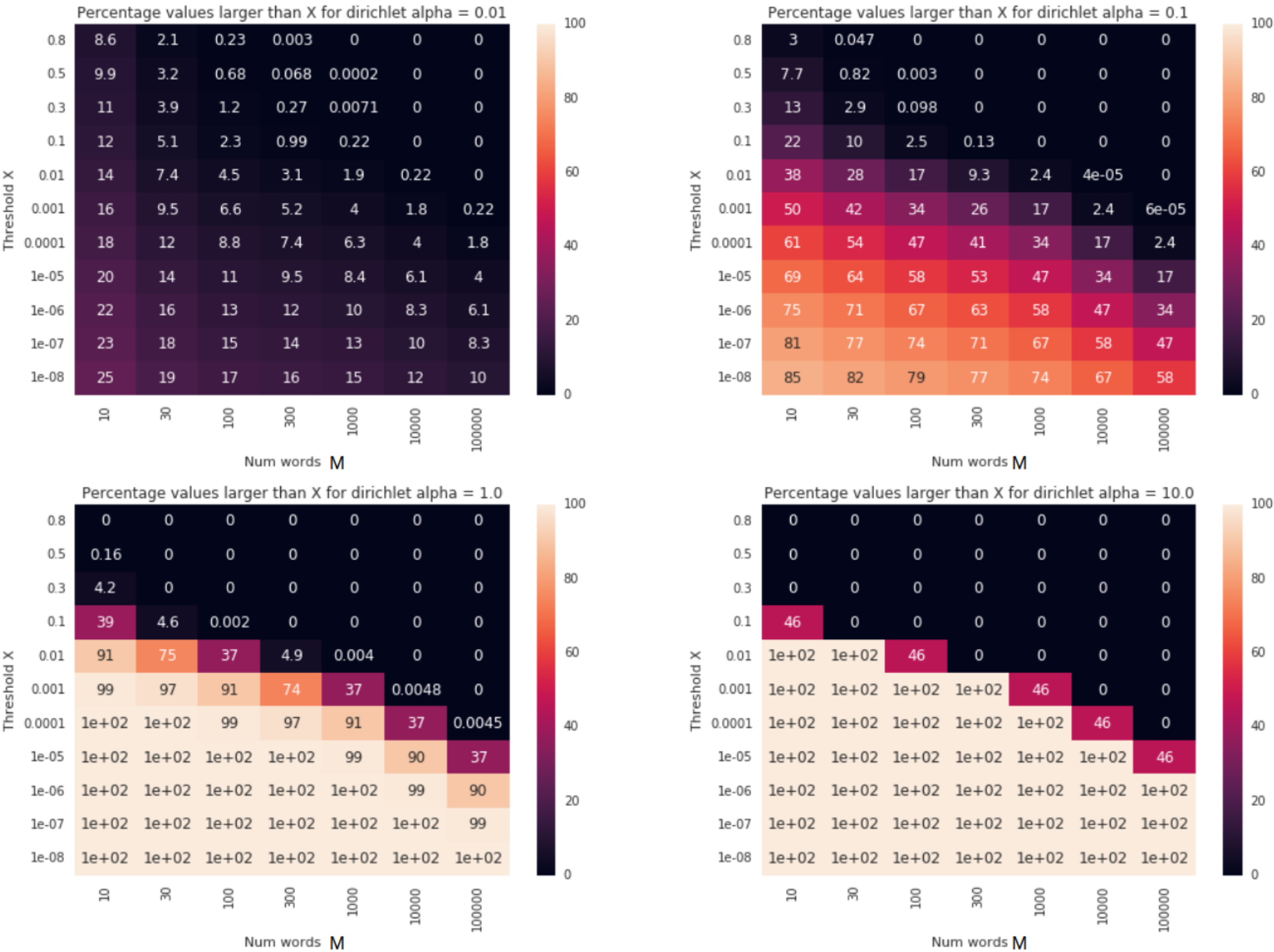}
\caption{Distribution of M-class discrete distributions sampled from a Dirichlet prior. Larger concentration parameters result in close to uniform distributions, while low values result in sparse or long-tail distributions.}
\label{fig:dir_conc}
\end{center}
\end{figure*}

%%%%%%%%%%%%%%%%%%%%%%%%%%%%%%%%%%%%%%%%%%%%%%%%%%%%%%%%%%%%%%%%%%%%%%%%%%%5

\section{Additional Synthetic Experiments}\label{app:synthetic}
See \cref{fig:argmax_0_01,fig:kl_0_01,fig:argmax_1}.

\begin{figure*}
\begin{center}
\includegraphics[scale=0.16]{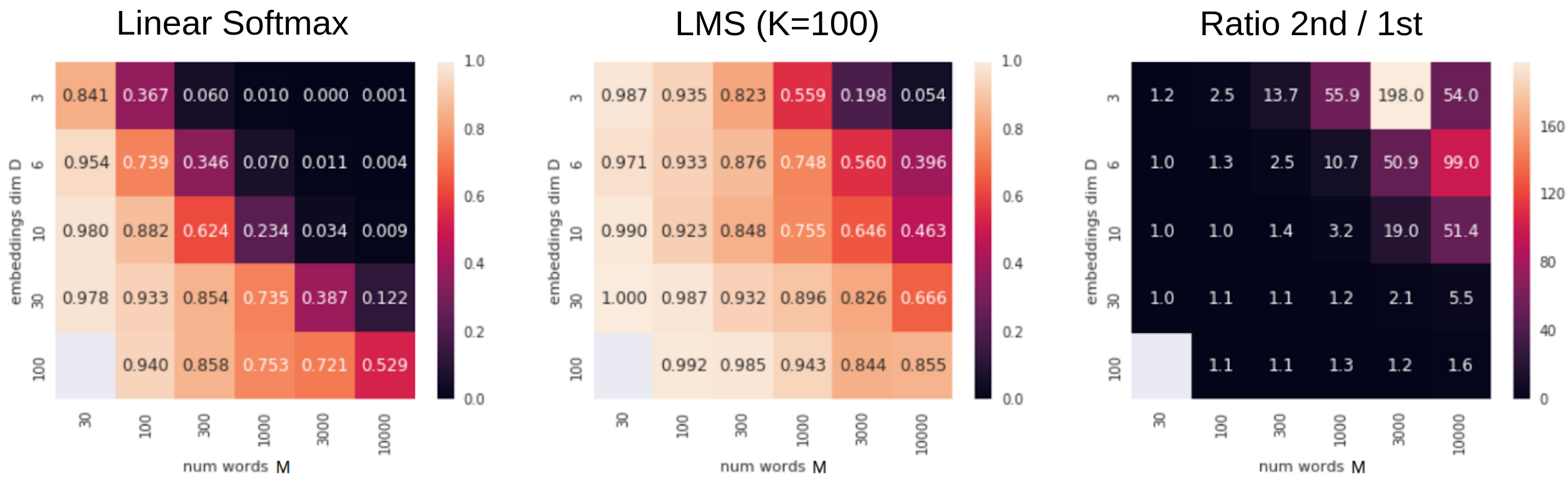}
\includegraphics[scale=0.16]{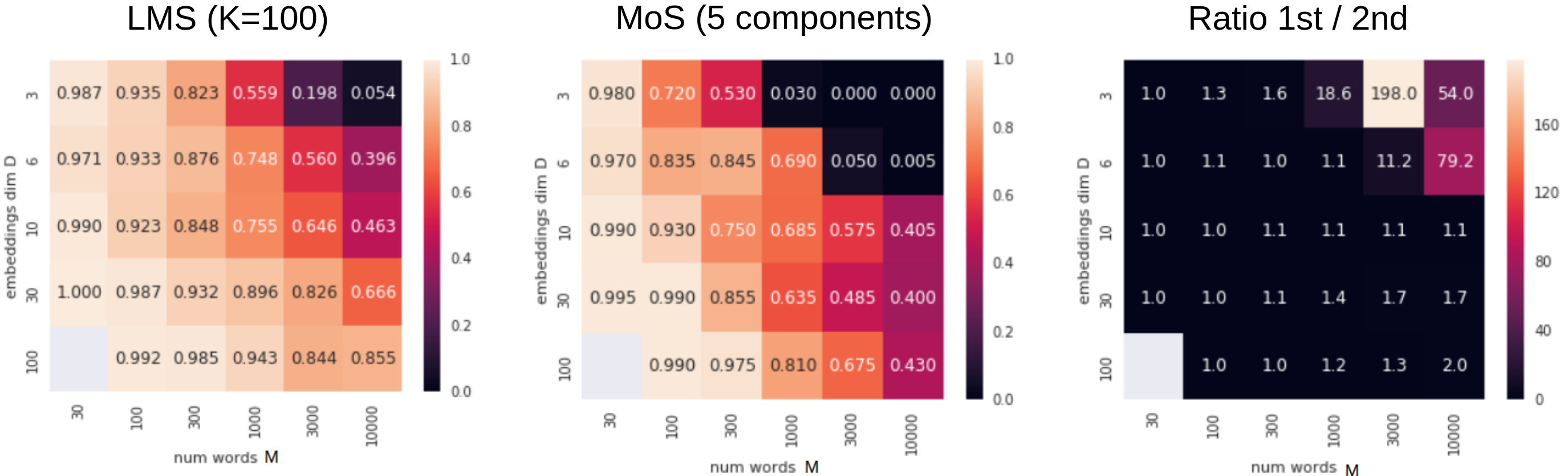}
\caption{Percentage of contexts $j$ for which the modes of true and parametric distributions match, i.e $\arg\max_{i} P^*(x_i|c_j) = \arg\max_{i} Q_{\Theta}(x_i|c_j)$. Higher the better. Dirichlet concentration $\alpha = 0.01$. }
\label{fig:argmax_0_01}
\end{center}
\end{figure*}

\begin{figure*}
\begin{center}
\includegraphics[scale=0.164]{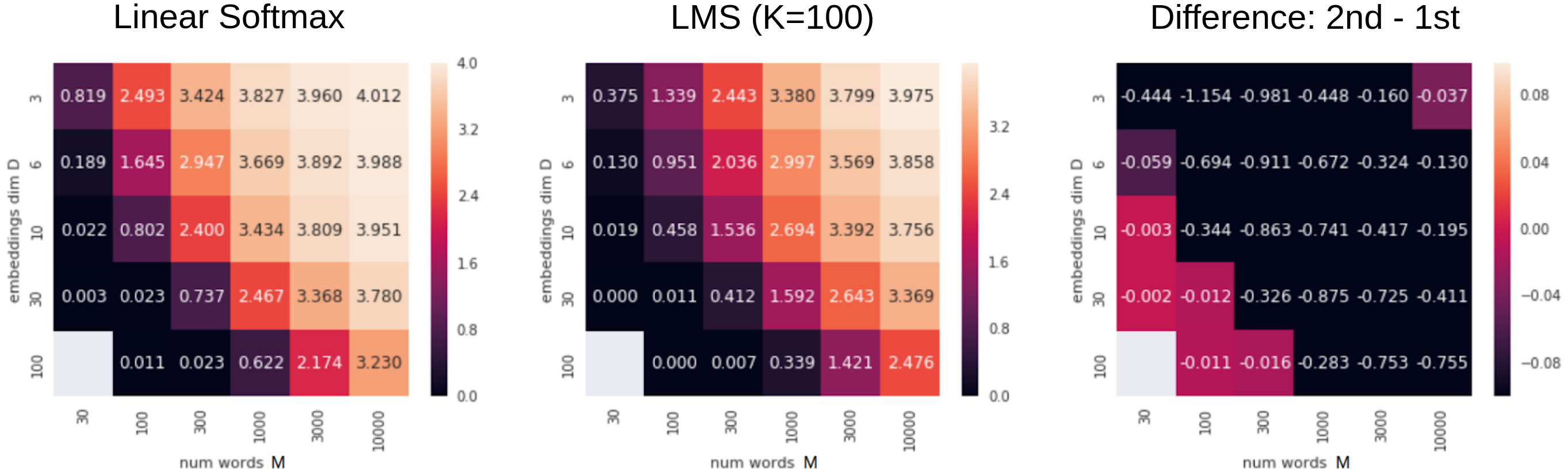}
\includegraphics[scale=0.16]{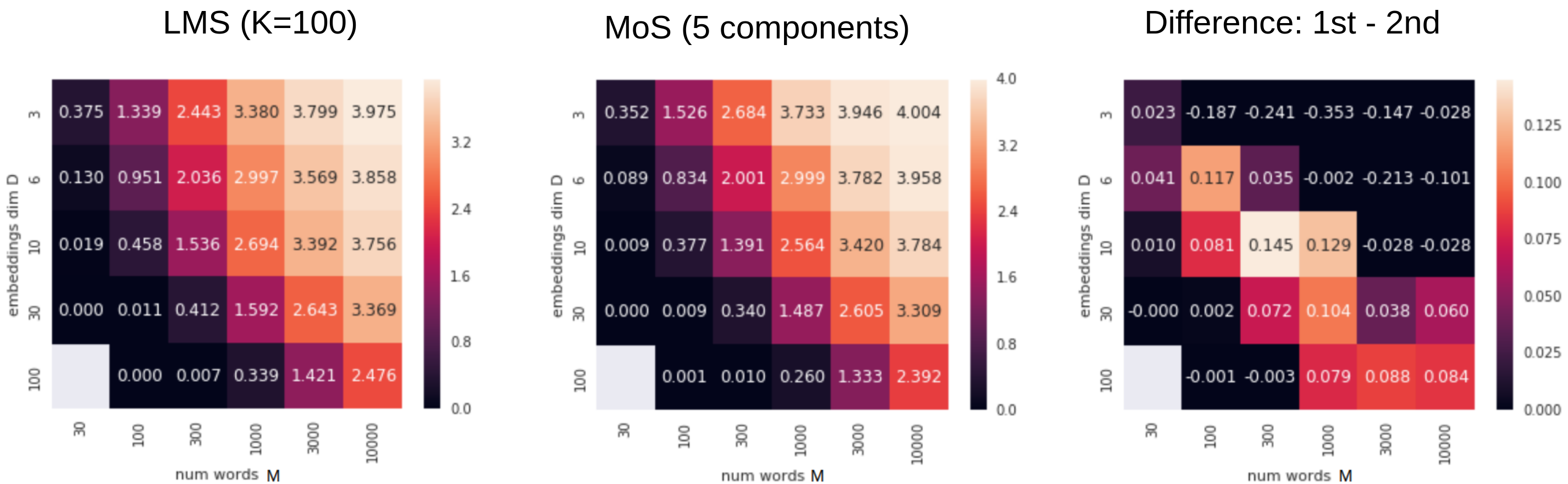}
\caption{Average $KL(P^* || Q_{\Theta})$ (across all contexts). Lower the better. Dirichlet concentration $\alpha = 0.01$. }
\label{fig:kl_0_01}
\end{center}
\end{figure*}

\begin{figure*}
\begin{center}
\includegraphics[scale=0.15]{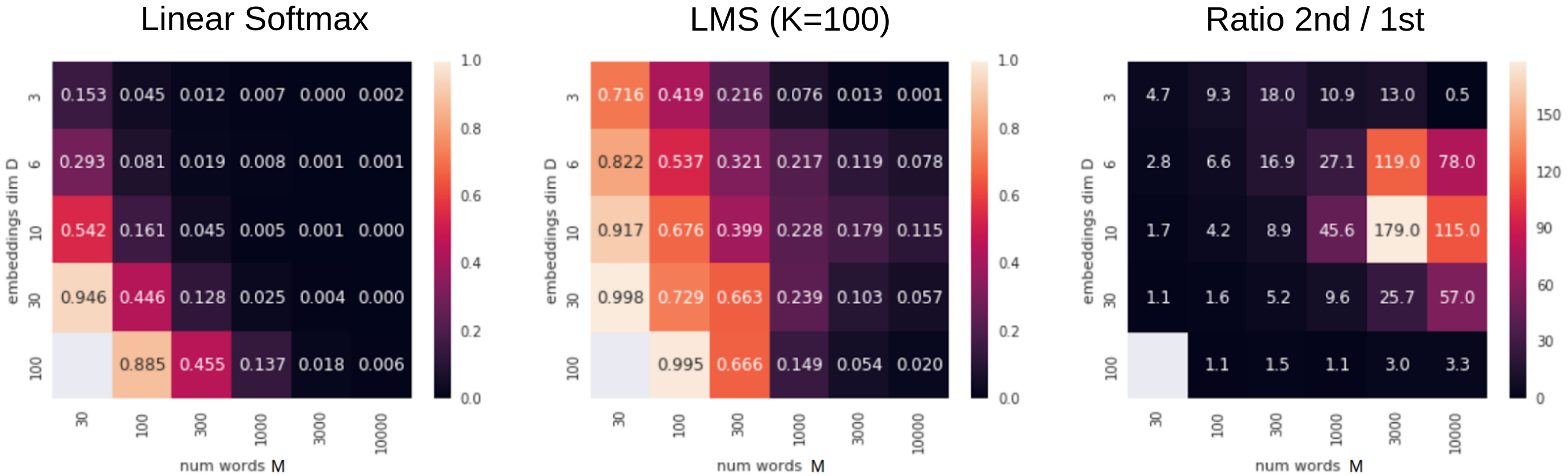}
\includegraphics[scale=0.15]{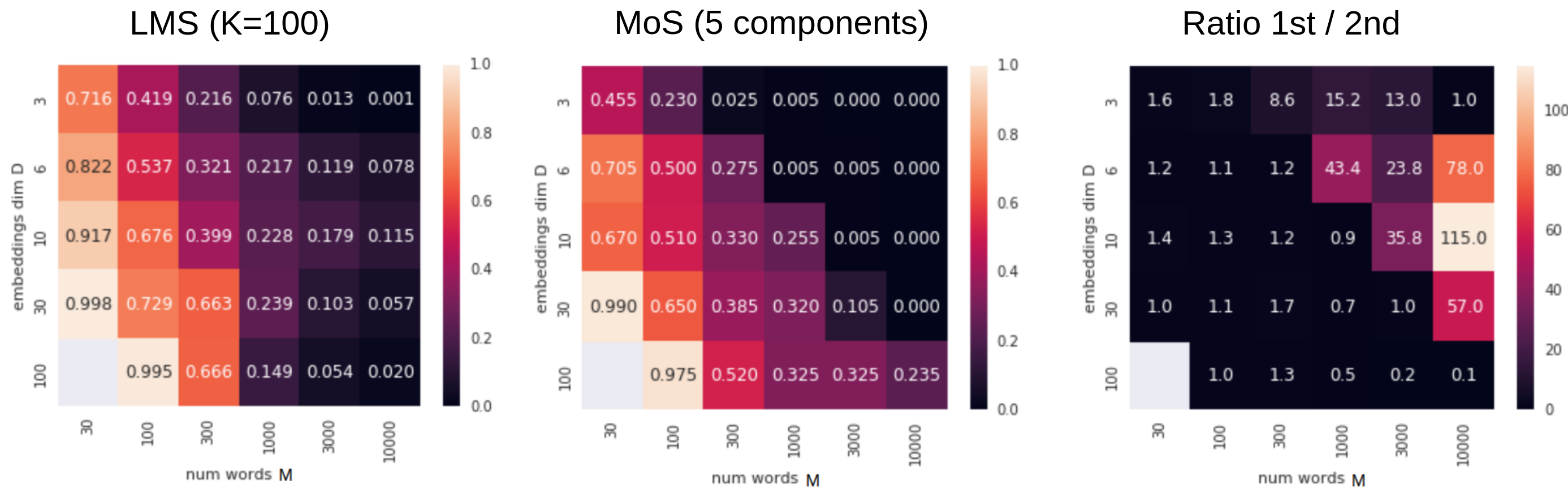}
\caption{Percentage of contexts $j$ for which the modes of true and parametric distributions match, i.e $\arg\max_{i} P^*(x_i|c_j) = \arg\max_{i} Q_{\Theta}(x_i|c_j)$. Higher the better. Dirichlet concentration $\alpha = 1$. }
\label{fig:argmax_1}
\end{center}
\end{figure*}